\documentclass{article}
\usepackage{amsfonts}
\usepackage{amsmath}

\newtheorem{theorem}{Theorem}

\newtheorem{lemma}[theorem]{Lemma}

\newenvironment{proof}[1][Proof]{\noindent\textbf{#1.} }{\ \rule{0.5em}{0.5em}}
\input{tcilatex}
\begin{document}

\title{Excess risk bounds for multitask learning with trace norm
regularization}
\author{Andreas Maurer \\
Adalbertstr. 55, D-80799 M\"{u}nchen, Germany \\
\emph{am@andreas-maurer.eu} \and Massimiliano Pontil \\
Department of Computer Science \\
University College London \\
Malet Place London WC1E, UK\\
\emph{m.pontil@cs.ucl.ac.uk} }
\maketitle

\begin{abstract}
Trace norm regularization is a popular method of multitask learning. We give
excess risk bounds with explicit dependence on the number of tasks, \ the
number of examples per task and properties of the data distribution. The
bounds are independent of the dimension of the input space, which may be
infinite as in the case of reproducing kernel Hilbert spaces. A byproduct of
the proof are bounds on the expected norm of sums of random positive
semidefinite matrices with subexponential moments.
\end{abstract}

\section{Introduction}

A fundamental limitation of supervised learning is the cost incurred by the
preparation of the large training samples required for good generalization.
A potential remedy is offered by multi-task learning: in many cases, while
individual sample sizes are rather small, there are samples to represent a
large number of learning tasks, which share some constraining or generative
property. This common property can be estimated using the entire collection
of training samples, and if this property is sufficiently simple it should
allow better estimation of the individual tasks from small individual
samples. 

The machine learning community has tried multi-task learning for many years
(see \cite{Zhang 2005,AEP,Caruana 1998,Cesa,Evgeniou 2005,Lounici 2011,Maurer 2005,Thrun 1998}, contributions and references
therein), but there are few theoretical investigations which clearly expose
the conditions under which multi-task learning is preferable to independent
learning. Following the seminal work of Baxter (\cite{Baxter 1998,Baxter 2000}) several authors have given generalization and performance
bounds under different assumptions of task-relatedness. In this paper we
consider multi-task learning with trace-norm regularization (TNML), a
technique for which efficient algorithms exist and which has been
successfully applied many times (see e.g. \cite{Amit 2007,AEP,Evgeniou
2005,Harchaoui}).\bigskip 

In the learning framework considered here the inputs live in a separable
Hilbert space $H$, which may be finite or infinite dimensional, and the
outputs are real numbers. For each of $T$ tasks an unknown input-output
relationship is modeled by a distribution $\mu _{t}$ on $H\times \mathcal{%
\mathbb{R}
}$ , with $\mu _{t}\left( X,Y\right) $ being interpreted as the probability
of observing the input-output pair $\left( X,Y\right) $. We assume bounded
inputs, for simplicity $\left\Vert X\right\Vert \leq 1$, where we use $%
\left\Vert \cdot\right\Vert $ and $\left\langle \cdot,\cdot\right\rangle $ to denote
euclidean norm and inner product in $H$ respectively.

A predictor is specified by a weight vector $w\in H$ which predicts
the output $\left\langle w,x\right\rangle $ for an observed input $x\in H$.
If the observed output is $y$ a loss $\ell \left( \left\langle
w,x\right\rangle ,y\right) $ is incurred, where $\ell $ is a fixed loss
function on $%
\mathbb{R}
^{2}$, assumed to have values in $\left[ 0,1\right] $, with $\ell \left(
\cdot ,y\right) $ being Lipschitz with constant $L$ for each $y\in \mathcal{%
\mathbb{R}
}$. The expected loss or risk of weight vector $w$ in the context of task $t$
is thus%
\begin{equation*}
R_{t}\left( w\right) =\mathbb{E}_{\left( X,Y\right) \sim \mu _{t}}\left[
\ell \left( \left\langle w,X\right\rangle ,Y\right) \right] .
\end{equation*}%
The choice of a weight vector $w_{t}$ for each task $t$ is equivalent to
the choice of a linear map $W:H\rightarrow 
\mathbb{R}
^{T}$, with $\left( Wx\right) _{t}=\left\langle x,w_{t}\right\rangle $. We
seek to choose $W$ so as to (nearly) minimize the total average risk $%
R\left( W\right) $ defined by%
\begin{equation*}
R\left( W\right) =\frac{1}{T}\sum_{t=1}^{T}\mathbb{E}_{\left( X,Y\right)
\sim \mu _{t}}\left[ \ell \left( \left\langle w_{t},X\right\rangle ,Y\right) %
\right] .
\end{equation*}

Since the $\mu _{t}$ are unknown, the minimization is based on a finite
sample of observations, which for each task $t$ is modelled by a vector $%
\mathbf{Z}^{t}$ of $n$ independent random variables $\mathbf{Z}^{t}=\left(
Z_{1}^{t},\dots,Z_{n}^{t}\right) $, where each $Z_{i}^{t}=\left(
X_{i}^{t},Y_{i}^{t}\right) $ is distributed according to $\mu _{t}$. For
most of this paper we make the simplifying assumption that all the samples
have the same size $n$. With an appropriate modification of the algorithm
defined below this assumption can be removed (see Remark \ref{Remark
different sample sizes} below). In a similar way the definition of $R\left(
W\right) $ can be replaced by a weighted average which attribute greater
weight to tasks which are considered more important. The entire multi-sample 
$\left( \mathbf{Z}^{1},\dots,\mathbf{Z}^{T}\right) $ is denoted by $\mathbf{%
\bar{Z}}$.

A classical and intuitive learning strategy is empirical risk minimization.
One decides on a constraint set $\mathcal{W\subseteq L}\left( H,%
\mathbb{R}
^{T}\right) $ for candidate maps and solves the problem 
\begin{equation*}
\hat{W}\left( \mathbf{\bar{Z}}\right) =\arg \min_{W\in \mathcal{W}}\hat{R}%
\left( W,\mathbf{\bar{Z}}\right) ,
\end{equation*}%
where the average empirical risk $\hat{R}\left( W,\mathbf{\bar{Z}}\right) $
is defined as%
\begin{equation*}
\hat{R}\left( W,\mathbf{\bar{Z}}\right) =\frac{1}{T}\sum_{t=1}^{T}\frac{1}{n}%
\sum_{i=1}^{n}\ell \left( \left\langle w_{t},X_{i}^{t}\right\rangle
,Y_{i}^{t}\right) .
\end{equation*}%
If the candidate set $\mathcal{W}$ has the form $\mathcal{W=}\left\{
x\mapsto Wx:\left( Wx\right) _{t}=\left\langle x,w_{t}\right\rangle
,w_{t}\in \mathcal{B}\right\} $ where $\mathcal{B}\subseteq H$ 
is some candidate set of vectors, then this is equivalent to
single task learning, solving for each task the problem%
\begin{equation*}
w_{t}\left( \mathbf{Z}_{t}\right) =\arg \min_{w\in \mathcal{B}}\frac{1}{n}%
\sum_{i=1}^{n}\ell \left( \left\langle w,X_{i}^{t}\right\rangle
,Y_{i}^{t}\right) .
\end{equation*}%
For proper multi-task learning the set $\mathcal{W}$ is chosen such that for
a map $W$ membership in $\mathcal{W}$ implies some mutual dependence between the 
vectors $w_{t}$.

A good candidate set $\mathcal{W}$ must fulfill two requirements: it must be
large enough to contain maps with low risk and small enough that we can find
such maps from a finite number of examples. The first requirement means that
the risk of the best map $W^{\ast }$ in the set, 
\begin{equation*}
W^{\ast }=\arg \min_{W\in \mathcal{W}}R\left( W\right) ,
\end{equation*}%
is small. This depends on the set of tasks at hand and is largely a matter
of domain knowledge. The second requirement is that the risk of the operator
which we find by empirical risk minimization, $\hat{W}\left( \mathbf{\bar{Z}}%
\right) $, is not too different from the risk of $W^{\ast }$, so that the
excess risk 
\begin{equation*}
R\left( \hat{W}\left( \mathbf{\bar{Z}}\right) \right) -R\left( W^{\ast
}\right) 
\end{equation*}%
is small. Bounds on this quantity are the subject of this paper, and, as $%
R\left( \hat{W}\left( \mathbf{\bar{Z}}\right) \right) $ is a random
variable, they can only be expected to hold with a certain probability.

For multitask learning with trace-norm regularization (TNML) we suppose that 
$\mathcal{W}$ is defined in terms of the trace-norm%
\begin{equation}
\mathcal{W=}\left\{ W\in 
\mathbb{R}
^{dT}:\left\Vert W\right\Vert _{1}\leq B\sqrt{T}\right\} ,
\label{Main constraint}
\end{equation}%
where $\left\Vert W\right\Vert _{1}=tr\left( \left( W^{\ast }W\right)
^{1/2}\right) $ and $B>0$ is a regularization constant. The factor $\sqrt{T}$
is an important normalization which we explain below. We will prove

\begin{theorem}
\label{Theorem Main}(i) For $\delta >0$ with probability at least $1-\delta $
in $\mathbf{\bar{Z}}$%
\begin{equation*}
R\left( \hat{W}\right) -R\left( W^{\ast }\right) \leq 2LB\left( \sqrt{\frac{%
\left\Vert C\right\Vert _{\infty }}{n}}+5\sqrt{\frac{\ln \left( nT\right) +1%
}{nT}}\right) +\sqrt{\frac{2\ln \left( 2/\delta \right) }{nT}},
\end{equation*}%
where $\left\Vert .\right\Vert _{\infty }$ is the operator, or spectral
norm, and $C$ is the task averaged, uncentered data covariance operator 
\begin{equation*}
\left\langle Cv,w\right\rangle =\frac{1}{T}\sum_{t=1}^{T}\mathbb{E}_{\left(
X,Y\right) \sim \mu _{t}}\left\langle v,X\right\rangle \left\langle
X,w\right\rangle \text{, for }w,v\in H.
\end{equation*}%
(ii) Also with probability $1-\delta $ in $\mathbf{\bar{Z}}$%
\begin{equation*}
R\left( \hat{W}\right) -R\left( W^{\ast }\right) \leq 2LB\left( \sqrt{\frac{%
\left\Vert \hat{C}\right\Vert _{\infty }}{n}}+\sqrt{\frac{2\left( \ln \left(
nT\right) +1\right) }{nT}}\right) +\sqrt{\frac{8\ln \left( 3/\delta \right) 
}{nT}},
\end{equation*}%
with $\hat{C}$ being the task averaged, uncentered empirical covariance
operator%
\begin{equation*}
\left\langle \hat{C}v,w\right\rangle =\frac{1}{nT}\sum_{t=1}^{T}%
\sum_{i=1}^{n}\left\langle v,X_{i}^{t}\right\rangle \left\langle
X_{i}^{t},w\right\rangle \text{, for }w,v\in H.
\end{equation*}
\end{theorem}

\noindent Remarks:
\begin{enumerate}
\item The first bound is distribution dependent, the second data-dependent.

\item Suppose that for an operator $W$ all $T$ column vectors $w_{t}$ are
equal to a common vector $w$, as might be the case if all the tasks $T$ are
equivalent. In this case increasing the number of tasks should not increase
the regularizer. Since then $\left\Vert W\right\Vert _{1}=\sqrt{T}\left\Vert
w\right\Vert $ we have chosen the factor $\sqrt{T}$ in (\ref{Main constraint}%
). It allows us to consider the limit $T\rightarrow \infty $ for a fixed
value of $B$.

\item In the limit $T\rightarrow \infty $ the bounds become 
\begin{equation*}
2LB\sqrt{\frac{\left\Vert C\right\Vert _{\infty }}{n}}\text{ or }2LB\sqrt{%
\frac{\left\Vert \hat{C}\right\Vert _{\infty }}{n}}\text{ respectively.}
\end{equation*}%
The limit is finite and it is approached at a rate of $\sqrt{\ln \left(
T\right) /T}$.

\item If the mixture of data distributions is supported on a one dimensional
subspace then $\left\Vert C\right\Vert _{\infty }=\mathbb{E}\left\Vert
X\right\Vert ^{2}$ and the bound is always worse than standard bounds for
single task learning as in \cite{Bartlett 2002}. The situation is similar if
the distribution is supported on a very low dimensional subspace. Thus, if
learning is already easy, TNML will bring no benefit.

\item If the mixture of data distributions is uniform on an $M$-dimensional
unit sphere in $H$ then $\left\Vert C\right\Vert _{\infty }=1/M$ and the
corresponding term in the bound becomes small. Suppose now that for $%
W=\left( w_{1},\dots,w_{T}\right) $ the $w_{t}$ all are constrained to be unit
vectors lying in some $K$-dimensional subspace of $H$, as might be the
solution returned by a method of subspace learning \cite{Zhang 2005}. If we
choose $B=K^{1/2}$ then $W\in \mathcal{W}$, and our bound also applies. This
subspace corresponds to the property shared shared among the tasks. The cost
of its estimation vanishes in the limit $T\rightarrow \infty $ and the bound
becomes%
\begin{equation*}
2L\sqrt{\frac{K}{nM}}.
\end{equation*}%
$K$ is proportional to the number of bits needed to communicate the utilized
component of an input vector, given knowledge of the common subspace. $M$ is
proportional to the number of bits to communicate an entire input vector. In
this sense the quantity $K/M$ can be interpreted as the ratio of the
utilized information $K$ to the available information $M$, as in \cite%
{Maurer 2006}. If $T$ and $M$ are large and $K$ is small the excess risk can
be very small even for small sample sizes $m$. Thus, if learning is
difficult (due to data of intrinsically high dimension) and the
approximation error is small, then TNML is superior to single task learning.

\item An important example of the infinite dimensional case is given when $H$
is the reproducing kernel Hilbert space $H_{\kappa }$ generated by a
positive semidefinite kernel $\kappa :Z\times Z\rightarrow {\mathbb{R}}$
where $Z$ is a set of inputs. This setting is important because it allows to
learn large classes of nonlinear functions. By the representer theorem for
matrix regularizers \cite{AMP} empirical risk minimization within the
hypothesis space $\mathcal{W}$ reduces to a finite dimensional problem in $%
nT^{2}$ variables. 

\item \label{Remark different sample sizes}The assumption of equal sample
sizes for all tasks is often violated in practice. Let $n_{t}$ be the number
of examples available for the $t$-th task. The resulting imbalance can be
compensated by a modification of the regularizer, replacing $\left\Vert
W\right\Vert _{1}$ by a weighted trace norm $\left\Vert SW\right\Vert _{1}$,
where the diagonal matrix $S=\left( s_{1},\dots,s_{T}\right) $ weights the $t$%
-th task with%
\begin{equation*}
s_{t}=\sqrt{\frac{1}{n_{t}T}\sum_{r}n_{r}}.
\end{equation*}%
where $n_{t}$ is the size of the sample available for the $t$-th task. With
this modification the Theorem holds with the average sample size $\bar{n}%
=\left( 1/T\right) \sum n_{t}$ in place of $n$. In Section \ref{Section
Proof of theorem 1} we will prove this result, which then reduces to Theorem %
\ref{Theorem Main} when all the sample sizes are equal.
\end{enumerate}

The proof of Theorem \ref{Theorem Main} is based on the well established
method of Rademacher averages \cite{Bartlett 2002}\ and more recent advances
on tail bounds for sums of random matrices, drawing heavily on the work of
Ahlswede and Winter \cite{Ahlswede Winter}, Oliveira \cite{Oliveira 2010}
and Tropp \cite{Tropp 2010}. In this context two auxiliary results are
established (Theorem \ref{Theorem Main Tool} and Theorem \ref{Theorem
Oliveira} below), which may be of independent interest.

\section{Earlier work.\label{Section earlier work}}

The foundations to a theoretical understanding of multi-task learning were
laid by J. Baxter in \cite{Baxter 2000}, where covering numbers are used to
expose the potential benefits of multi-task and transfer learning. In \cite%
{Zhang 2005} Rademacher averages are used to give excess risk bounds for a
method of multi-task subspace learning. Similar results are obtained in \cite%
{Maurer 2005}. \cite{Ben-David 2003} uses a special assumption of
task-relatedness to give interesting bounds not on the average, but the
maximal risk over the tasks.

A lot of important work on trace norm regularization concerns matrix
completion, where a matrix is only partially observed and approximated (or
under certain assumptions even reconstructed) by a matrix of small trace
norm (see e.g. \cite{Candes 2009}, \cite{Shamir 2011} and references
therein). For $H=%
\mathbb{R}
^{d}$ and $T\times d$-matrices, this is somewhat related to the situation
considered here, if we identify the tasks with the columns of the matrix in
question, the input marginal as the uniform distribution supported on the
basis vectors of $%
\mathbb{R}
^{d}$ and the outputs as defined by the matrix values themselves, without or
with the addition of noise. One essential difference is that matrix
completion deals with a known and particularly simple input distribution,
which makes it unclear how bounds for matrix completion can be converted to
bounds for multitask learning. On the other hand our bounds cannot be
directly applied to matrix completion, because they assume a fixed number of
revealed entries for each column.

Multitask learning is considered in \cite{Lounici 2011}, where special
assumptions (coordinate-sparsity of the solution, restricted eigenvalues)
are used to derive fast rates and the recovery of shared features. Such
assumptions are absent in this paper, and \cite{Lounici 2011} also considers
a different regularizer.

\cite{Maurer 2006} and \cite{KakadeEtAl 2012} seem to be most closely
related to the present work. In \cite{Maurer 2006} the general form of the
bound is very similar to Theorem \ref{Theorem Main}. The result is dimension
independent, but it falls short of giving the rate of $\sqrt{\ln \left(
T\right) /T}$ in the number of tasks. Instead it gives $T^{-1/4}$.

\cite{KakadeEtAl 2012} introduces a general and elegant method to derive
bounds for learning techniques which employ matrix norms as regularizers.
For $H=%
\mathbb{R}
^{d}$ and applied to multi task learning and the trace-norm a data-dependent
bound is given whose dominant term reads as (omitting constants and
observing $\left\Vert W\right\Vert _{1}\leq B\sqrt{T}$)%
\begin{equation}
LB\sqrt{\max_{i}\left\Vert \hat{C}_{i}\right\Vert _{\infty }\frac{\ln \min
\left\{ T,d\right\} }{n}},  \label{Kakade bound}
\end{equation}%
where the matrix $\hat{C}_{i}$ is the empirical covariance of the data for
all tasks observed in the $i$-th observation%
\begin{equation*}
\hat{C}_{i}v=\frac{1}{T}\sum_{t}\left\langle v,X_{i}^{t}\right\rangle
X_{i}^{t}.
\end{equation*}%
The bound (\ref{Kakade bound}) does not paint a clear picture of the role of
the number of tasks $T$. Using Theorem \ref{Theorem Oliveira} below we can
estimate its expectation and convert it into the distribution dependent
bound with dominant term%
\begin{equation}
LB\sqrt{\ln \min \left\{ T,d\right\} }\left( \sqrt{\frac{\left\Vert
C\right\Vert _{\infty }}{n}}+\sqrt{\frac{6\ln \left( 24nT^{2}\right) +1}{nT}}%
\right) .  \label{Kakade bound distribution dependent}
\end{equation}

This is quite similar to Theorem \ref{Theorem Main} (i). Because (\ref%
{Kakade bound}) is hinged on the $i$-th observation it is unclear how it can
be modified for unequal sample sizes for different tasks. The principal
disadvantage of (\ref{Kakade bound}) however is that it diverges in the
simultaneous limit $d,T\rightarrow \infty $.

\section{Notation and Tools\label{Section Notation and Tools}}

The letters $H$, $H^{\prime }$, $H^{\prime \prime }$ will denote finite or
infinite dimensional separable real Hilbert spaces. 

For a linear map $A:H\rightarrow H^{\prime }$ we denote the adjoint with $%
A^{\ast }$, the range by $Ran\left( A\right) $ and the null space by $%
Ker\left( A\right) $. $A$ is called compact if the image of the open unit
ball of $H$ under $A$ is pre-compact (totally bounded) in $H^{\prime }$. If $%
Ran\left( H\right) $ is finite dimensional then $A$ is compact, finite
linear combinations of compact linear maps and products with bounded linear
maps are compact. A linear map $A:H\rightarrow H$ is called an operator and
self-adjoint if $A^{\ast }=A$ and nonnegative (or positive) if it is
self-adjoint and $\left\langle Ax,x\right\rangle \geq 0$ (or $\left\langle
Ax,x\right\rangle >0$) for all $x\in H$, $x\neq 0$, in which case we write $%
A\succeq 0\,$(or $A\succ 0$). We use "$\preceq $" to denote the order
induced by the cone of nonnegative operators.

For linear $A:H\rightarrow H^{\prime }$ and $B:H^{\prime }\rightarrow
H^{\prime \prime }$ the product $BA:H\rightarrow H^{\prime \prime }$ is
defined by $\left( BA\right) x=B\left( Ax\right) $. Then $A^{\ast
}A:H\rightarrow H$ is always a nonnegative operator. We use $\left\Vert
A\right\Vert _{\infty }$ for the norm $\left\Vert A\right\Vert _{\infty
}=\sup \left\{ \left\Vert Ax\right\Vert :\left\Vert x\right\Vert \leq
1\right\} $. We generally assume $\left\Vert A\right\Vert _{\infty }<\infty $%
. 

If $A$ is a compact and self-adjoint operator then there exists an
orthonormal basis $e_{i}$ of $H$ and real numbers $\lambda _{i}$ satisfying $%
\left\vert \lambda _{i}\right\vert \rightarrow 0$ such that 
\begin{equation*}
A=\sum_{i}\lambda _{i}Q_{e_{i}},
\end{equation*}%
where $Q_{e_{i}}$ is the operator defined by $Q_{e_{i}}x=\left\langle
x,e_{i}\right\rangle e_{i}$. The $e_{i}$ are eigenvectors and the $\lambda
_{i}$ eigenvalues of $A$. If $f$ is a real function defined on a set
containing all the $\lambda _{i}$ a self-adjoint operator $f\left( A\right) $
is defined by 
\begin{equation*}
f\left( A\right) =\sum_{i}f\left( \lambda _{i}\right) Q_{e_{i}}.
\end{equation*}%
$f\left( A\right) $ has the same eigenvectors as $A$ and eigenvalues $%
f\left( A\right) $. In the sequel self-adjoint operators are assumed to be
either compact or of the form $f\left( A\right) $ with $A$ compact (we will
encounter no others), so that there always exists a basis of eigenvectors. A
self-adjoint operator is nonnegative (positive) if all its eigenvalues are
nonnegative (positive). If $A$ is positive then $\ln \left( A\right) $
exists and has the property $\ln \left( A\right) \preceq \ln \left( B\right) 
$ whenever $B$ is positive and $A\preceq B$. This property of operator
monotonicity will be tacitly used in the sequel.

We write $\lambda _{\max }\left( A\right) $ for the largest eigenvalue (if
it exists), and for nonnegative operators $\lambda _{\max }\left( \cdot\right) $
always exists and coincides with the norm $\left\Vert \cdot\right\Vert _{\infty }
$. 

A linear subspace $M\subseteq H$ is called invariant under $A$ if $%
AM\subseteq M$. For a linear subspace $M\subseteq H$ we use $M^{\perp }$ to
denote the orthogonal complement $M^{\perp }=\left\{ x\in H:\left\langle
x,y\right\rangle =0,\forall y\in M\right\} $. For a selfadjoint operator $%
Ran\left( A\right) ^{\perp }=Ker\left( A\right) $. For a self-adjoint
operator $A$ on $H$ and an invariant subspace $M$ of $A$ the trace $tr_{M}A$
of $A$ relative to $M$ is defined 
\begin{equation*}
tr_{M}A=\sum_{i}\left\langle Ae_{i},e_{i}\right\rangle ,
\end{equation*}%
where $\left\{ e_{i}\right\} $ is a orthonormal basis of $M$. The choice of
basis does not affect the value of $tr_{M}$. For $M=H$ we just write $tr$
without subscript. The trace-norm of any linear map from $H$ to any Hilbert
space is defined as 
\begin{equation*}
\left\Vert A\right\Vert _{1}=tr\left( \left( A^{\ast }A\right) ^{1/2}\right)
.
\end{equation*}%
If $\left\Vert A\right\Vert _{1}<\infty $ then $A$ is compact. If $A$ is an
operator and $A\succeq 0$ then $\left\Vert A\right\Vert _{1}$ is simply the
sum of eigenvalues of $A$. In the sequel we will use Hoelder's inequality 
\cite{Bathia} for linear maps in the following form.

\begin{theorem}
\label{Theorem Hoelders inequality}Let $A$ and $B$ be two linear maps $%
H\rightarrow 
\mathbb{R}
^{T}$. Then $\left\vert tr\left( A^{\ast }B\right) \right\vert \leq
\left\Vert A\right\Vert _{1}\left\Vert B\right\Vert _{\infty }$.
\end{theorem}

\textbf{Rank-1 operators and covariance operators.} For $w\in H$ we define
an operator $Q_{w}$ by%
\begin{equation*}
Q_{w}v=\langle v,w\rangle w,{\rm~ for~}v\in H.
\end{equation*}%
In matrix notation this would be the matrix $ww^{\ast }$. It can also be
written as the tensor product $w\otimes w$. We apologize for the unusual
notation $Q_{w}$, but it will save space in many of the formulas below. The
covariance operators in Theorem \ref{Theorem Main} are then given by%
\begin{equation*}
C=\frac{1}{T}\sum_{t}\mathbb{E}_{\left( X,Y\right) \sim \mu _{t}}Q_{X}\text{
and }\hat{C}=\frac{1}{nT}\sum_{t,i}Q_{X_{i}^{t}}\text{.}
\end{equation*}%
Here and in the sequel the Rademacher variables $\sigma _{i}^{t}$ (or
sometimes $\sigma _{i}$) are uniformly distributed on $\left\{ 0,1\right\} $%
, mutually independent and independent of all other random variables, and $%
\mathbb{E}_{\sigma }$ is the expectation conditional on all other random
variables present. We conclude this section with two lemmata. Two numbers $%
p,q>1$ are called conjugate exponents if $1/p+1/q=1$.

\begin{lemma}
\label{Lemma conjugate exponents}(i) Let $p,q$ be conjugate exponents and $%
s,a\geq 0$. Then $\left( \sqrt{s+pa}-\sqrt{a}\right) ^{2}\geq s/q$. (ii) For 
$a,b>0$%
\begin{equation*}
\min_{q,p>1\text{ and }1/q+1/p=1}\sqrt{pa+qb}=\sqrt{a}+\sqrt{b}.
\end{equation*}%
(iii) and for $a,b>0$ we have $2\sqrt{ab}\leq \left( p-1\right) a+\left(
q-1\right) b$.
\end{lemma}

\begin{proof}
For conjugate exponents $p$ and $q$ we have $p-1=p/q$ and $q-1=q/p$.
Therefore $pa+qb-\left( \sqrt{a}+\sqrt{b}\right) ^{2}=\left( \sqrt{pa/q}-%
\sqrt{qb/p}\right) ^{2}\geq 0$, which proves (iii) and gives $\sqrt{pa+qb}%
\geq \sqrt{a}+\sqrt{b}$. Take $s=qb$, subtract $\sqrt{a}$ and square to get
(i). Set $p=1+\sqrt{b/a}$ and $q=1+\sqrt{a/b}$ to get (ii).
\end{proof}

\begin{lemma}
\label{Lemma partial integration}Let $a,c>0,b\geq 1$ and suppose the real
random variable $X\geq 0$ satisfies $\Pr \left\{ X>pa+s\right\} \leq b\exp
\left( -s/\left( cq\right) \right) $ for all $s\geq 0$ and all conjugate
exponents $p$ and $q$. Then%
\begin{equation*}
\sqrt{\mathbb{E}X}\leq \sqrt{a}+\sqrt{c\left( \ln b+1\right) }.
\end{equation*}
\end{lemma}

\begin{proof}
We use partial integration. 
\begin{eqnarray*}
\mathbb{E}X &\leq &pa+qc\ln b+\int_{qc\ln b}^{\infty }\Pr \left\{
X>pa+s\right\} ds \\
&\leq &pa+qc\ln b+b\int_{qc\ln b}^{\infty }e^{-s/\left( cq\right)
}ds=pa+q\left( c\ln b+1\right) .
\end{eqnarray*}%
Take the square root of both sides and use Lemma \ref{Lemma conjugate
exponents} (ii) to optimize in $p$ and $q$ to obtain the conclusion.
\end{proof}

\section{Sums of random operators}

In this section we prove two concentration results for sums of nonnegative
operators with finite dimensional ranges. The first (Theorem \ref{Theorem
Main Tool}) assumes only a weak form of boundedness, but it is strongly
dimension dependent. The second result (Theorem \ref{Theorem Oliveira}) is
the opposite. We will use the following important result of Tropp (Lemma 3.4
in \cite{Tropp 2010}), derived from Lieb's concavity theorem (see \cite%
{Bathia}, Section IX.6):

\begin{theorem}
\label{Lieb concavity}Consider a finite sequence $A_{k}$ of independent,
random, self-adjoint operators and a finite dimensional subspace $M\subseteq
H$ such that $A_{k}M\subseteq M$. Then for $\theta \in 
\mathbb{R}
$%
\begin{equation*}
\mathbb{E~}tr_{M}~\exp \left( \theta \sum_{k}A_{k}\right) \leq tr_{M}~\exp
\left( \sum_{k}\ln \mathbb{E}e^{\theta A_{k}}\right) .
\end{equation*}
\end{theorem}

A corollary suited to our applications is the following

\begin{theorem}
\label{Theorem Tropp}Let $A_{1},\dots,A_{N}$ be of independent, random,
self-adjoint operators on $H$ and let $M\subseteq H$ be a nontrivial, finite
dimensional subspace such that $Ran\left( A_{k}\right) \subseteq M$ a.s. for
all $k$.

(i) If $A_{k}\succeq 0$ a.s then 
\begin{equation*}
\mathbb{E~}\exp \left( \left\Vert \sum_{k}A_{k}\right\Vert \right) \leq \dim
\left( M\right) \exp \left( \lambda _{\max }\left( \sum_{k}\ln \mathbb{E}%
e^{A_{k}}\right) \right) .
\end{equation*}

(ii) If the $A_{k}$ are symmetrically distributed then 
\begin{equation*}
\mathbb{E~}\exp \left( \left\Vert \sum_{k}A_{k}\right\Vert \right) \leq
2\dim \left( M\right) \exp \left( \lambda _{\max }\left( \sum_{k}\ln \mathbb{%
E}e^{A_{k}}\right) \right) .
\end{equation*}
\end{theorem}

\begin{proof}
Let $A=\sum_{k}A_{k}$. Observe that  $M^{\perp }\subseteq Ker\left( A\right)
\cap \left( \cup _{k}Ker\left( A_{k}\right) \right) $, and that $M$ is a
nontrivial invariant subspace for $A$ as well as for all the $A_{k}$.

(i) Assume $A_{k}\succeq 0$. Then also $A\succeq 0$. Since $M^{\perp
}\subseteq Ker\left( A\right) $ there is $x_{1}\in M$ with $\left\Vert
x_{1}\right\Vert =1$ and $Ax_{1}=\left\Vert A\right\Vert x_{1}$ (this also
holds if $A=0$, since $M$ is nontrivial). Thus $e^{A}x_{1}=e^{\left\Vert
A\right\Vert }x_{1}$. Extending $x_{1}$ to a basis $\left\{ x_{i}\right\} $
of $M$ we get 
\begin{equation*}
e^{\left\Vert A\right\Vert }=\left\langle e^{A}x_{1},x_{1}\right\rangle \leq
\sum_{i}\left\langle e^{A}x_{i},x_{i}\right\rangle =tr_{M}e^{A}\text{.}
\end{equation*}%
Theorem \ref{Lieb concavity} applied to the matrices which represent $A_{k}$
restricted to the finite dimensional invariant subspace $M$ then gives 
\begin{eqnarray*}
\mathbb{E}e^{\left\Vert A\right\Vert } &\leq &\mathbb{E}~tr_{M}~e^{A}\leq
tr_{M}~\exp \left( \sum_{k}\ln \left( \mathbb{E}e^{A_{k}}\right) \right)  \\
&\leq &\dim \left( M\right) ~\exp \left( \lambda _{\max }\left( \sum_{k}\ln
\left( \mathbb{E}e^{A_{k}}\right) \right) \right) ,
\end{eqnarray*}%
where the last inequality results from bounding $tr_{M}$ by $\dim \left(
M\right) \lambda _{\max }$ and $\lambda _{\max }\left( \exp \left( \cdot\right)
\right) =\exp \left( \lambda _{\max }\left( \cdot\right) \right) $.

(ii) Assume that $A_{k}$ is symmetrically distributed. Then so is $A$. Since 
$M^{\perp }\subseteq Ker\left( A\right) $ there is $x_{1}\in M$ with $%
\left\Vert x_{1}\right\Vert =1$ and either $Ax_{1}=\left\Vert A\right\Vert
x_{1}$ or $-Ax_{1}=\left\Vert A\right\Vert x_{1}$, so that either $%
e^{A}x_{1}=e^{\left\Vert A\right\Vert }x_{1}$ or $e^{-A}x_{1}=e^{\left\Vert
A\right\Vert }x_{1}$. Extending to a basis again gives 
\begin{equation*}
e^{\left\Vert A\right\Vert }\leq \left\langle e^{A}x_{1},x_{1}\right\rangle
+\left\langle e^{-A}x_{1},x_{1}\right\rangle \leq tr_{M}e^{A}+tr_{M}e^{-A}.
\end{equation*}%
By symmetric distribution we have 
\begin{equation*}
\mathbb{E}e^{\left\Vert A\right\Vert }\leq tr_{M}\left( \mathbb{E}e^{A}+%
\mathbb{E}e^{-A}\right) \leq 2\mathbb{E}~tr_{M}~e^{A}\text{.}
\end{equation*}%
Then continue as in case (ii).
\end{proof}

The following is our first technical tool.

\begin{theorem}
\label{Theorem Main Tool}Let $M\subseteq H$ be a subspace of dimension $d$
and suppose that $A_{1},\dots,A_{N}$ are independent random operators
satisfying $A_{k}\succeq 0$, $Ran\left( A_{k}\right) \subseteq M$ a.s. and 
\begin{equation}
\mathbb{E}A_{k}^{m}\preceq m!R^{m-1}\mathbb{E}A_{k}
\label{Subexponential assumption}
\end{equation}%
for some $R\geq 0$, all $m\in 
\mathbb{N}
$ and all $k\in \left\{ 1,\dots,N\right\} $. Then for $s\geq 0$ and conjugate
exponents $p$ and $q$%
\begin{equation*}
\Pr \left\{ \left\Vert \sum_{k}A_{k}\right\Vert _{\infty }>p\left\Vert 
\mathbb{E}\sum_{k}A_{k}\right\Vert _{\infty }+s\right\} \leq \dim \left(
M\right) e^{-s/\left( qR\right) }.
\end{equation*}%
Also%
\begin{equation*}
\sqrt{\mathbb{E}\left\Vert \sum_{k}A_{k}\right\Vert _{\infty }}\leq \sqrt{%
\left\Vert \mathbb{E}\sum_{k}A_{k}\right\Vert _{\infty }}+\sqrt{R\left( \ln
\dim \left( M\right) +1\right) }.
\end{equation*}
\end{theorem}

\begin{proof}
Let $\theta $ be any number satisfying $0\leq \theta <\frac{1}{R}$. From (%
\ref{Subexponential assumption}) we get for any $k\in \left\{
1,\dots,N\right\} $%
\begin{eqnarray*}
\mathbb{E}e^{\theta A_{k}} &=&I+\sum_{m=1}^{\infty }\frac{\theta ^{m}}{m!}%
\mathbb{E}A_{k}^{m}\preceq I+\sum_{m=1}^{\infty }\left( \theta R\right)
^{m}\left( R^{-1}\mathbb{E}A_{k}\right)  \\
&=&I+\frac{\theta }{1-R\theta }\mathbb{E}A_{k}\preceq \exp \left( \frac{%
\theta }{1-R\theta }\mathbb{E}A_{k}\right) .
\end{eqnarray*}%
Abbreviate $\mu =\left\Vert \mathbb{E}\sum_{k}A_{k}\right\Vert _{\infty }$
and let $r=s+p\mu $ and set 
\begin{equation*}
\theta =\frac{1}{R}\left( 1-\sqrt{\frac{\mu }{r}}\right) ,
\end{equation*}%
so that $0\leq \theta <1/R$. Applying the above inequality and the operator
monotonicity of the logarithm we get for all $k$ that $\ln \mathbb{E}\exp
\left( \theta A_{k}\right) \preceq \theta /\left( 1-R\theta \right) \mathbb{E%
}A_{k}$. Summing this relation over $k$ and passing to the largest
eigenvalue yields%
\begin{equation*}
\lambda _{\max }\left( \sum_{k}\ln \mathbb{E}e^{\theta A_{k}}\right) \leq 
\frac{\theta \mu }{1-R\theta }
\end{equation*}%
Now we combine Markov's inequality with Theorem \ref{Theorem Tropp} (i) and
the last inequality to obtain%
\begin{eqnarray*}
\Pr \left\{ \left\Vert \sum A_{k}\right\Vert _{\infty }\geq r\right\}  &\leq
&e^{-\theta r}\mathbb{E~}\exp \left( \theta \left\Vert
\sum_{k}A_{k}\right\Vert \right)  \\
&\leq &\dim \left( M\right) e^{-\theta r}\exp \left( \lambda _{\max }\left(
\sum_{k}\ln \mathbb{E}e^{\theta A_{k}}\right) \right)  \\
&\leq &\dim \left( M\right) \exp \left( -\theta r+\frac{\theta }{1-R\theta }%
\mu \right)  \\
&=&\dim \left( M\right) \exp \left( \frac{-1}{R}\left( \sqrt{r}-\sqrt{\mu }%
\right) ^{2}\right) .
\end{eqnarray*}%
By Lemma \ref{Lemma conjugate exponents} (i) $\left( \sqrt{r}-\sqrt{\mu }%
\right) ^{2}=\left( \sqrt{s+p\mu }-\sqrt{\mu }\right) ^{2}\geq s/q$, so this
proves the first conclusion. The second follows from the first and Lemma \ref%
{Lemma partial integration}.
\end{proof}

The next result and its proof are essentially due to Oliveira (\cite%
{Oliveira 2010}, Lemma 1, but see also \cite{Mendelson 2006}. We give a
slightly more general version which eliminates the assumption of identical
distribution and has smaller constants.

\begin{theorem}
\label{Theorem Oliveira}Let $A_{1},\dots,A_{N}$ be independent random
operators satisfying $0\preceq A_{k}\preceq I$ and suppose that for some $%
d\in 
\mathbb{N}
$%
\begin{equation}
\dim Span\left( Ran\left( A_{1}\right) ,\dots,Ran\left( A_{N}\right) \right)
\leq d  \label{Oliveira dimension condition}
\end{equation}%
almost surely. Then

(i)%
\begin{equation*}
\Pr \left\{ \left\Vert \sum_{k}\left( A_{k}-\mathbb{E}A_{k}\right)
\right\Vert _{\infty }>s\right\} \leq 4d^{2}\exp \left( \frac{-s^{2}}{%
9\left\Vert \sum_{k}\mathbb{E}A_{k}\right\Vert _{\infty }+6s}\right) .
\end{equation*}

(ii)%
\begin{equation*}
\Pr \left\{ \left\Vert \sum_{k}A_{k}\right\Vert _{\infty }>p\left\Vert 
\mathbb{E}\sum_{k}A_{k}\right\Vert _{\infty }+s\right\} \leq
4d^{2}e^{-s/\left( 6q\right) }
\end{equation*}

(iii)%
\begin{equation*}
\sqrt{\mathbb{E}\left\Vert \sum_{k}A_{k}\right\Vert _{\infty }}\leq \sqrt{%
\left\Vert \mathbb{E}\sum_{k}A_{k}\right\Vert _{\infty }}+\sqrt{6\left( \ln
\left( 4d^{2}\right) +1\right) }
\end{equation*}
\end{theorem}

In the previous theorem the subspace $M$ was deterministic and had to
contain the ranges of \textit{all} possible random realizations of the $A_{k}
$. By contrast the span appearing in (\ref{Oliveira dimension condition}) is
the random subspace spanned by a single random realization of the $A_{k}$.
If all the $A_{k}$ have rank one, for example, we can take $d=N$ and apply
the present theorem even if each $\mathbb{E}A_{k}$ has infinite rank. This
allows to estimate the empirical covariance in terms of the true covariance
for a bounded data distribution in an infinite dimensional space.

\begin{proof}
Let $0\leq \theta <1/4$ and abbreviate $A=\sum_{k}A_{k}$. A standard
symmetrization argument (see \cite{Ledoux Talagrand 1991}, Lemma 6.3) shows
that 
\begin{equation*}
\mathbb{E}e^{\theta \left\Vert A-\mathbb{E}A\right\Vert }\leq \mathbb{EE}%
_{\sigma }\exp \left( 2\theta \left\Vert \sum_{k}\sigma _{k}A_{k}\right\Vert
\right) ,
\end{equation*}%
where the $\sigma _{k}$ are Rademacher variables and $\mathbb{E}_{\sigma }$
is the expectation conditional on the $A_{1},\dots,A_{N}$. For fixed $%
A_{1},\dots,A_{N}$ let $M$ be the linear span of their ranges, which has
dimension at most $d$ and also contains the ranges of the symmetrically
distributed operators $2\theta \sigma _{k}A_{k}$. Invoking Theorem \ref%
{Theorem Tropp} (ii) we get%
\begin{eqnarray*}
\mathbb{E}_{\sigma }\exp \left( 2\theta \left\Vert \sum_{k}\sigma
_{k}A_{k}\right\Vert \right)  &\leq &2d\exp \left( \lambda _{\max }\left(
\sum_{k}\ln \mathbb{E}_{\sigma }e^{2\theta \sigma _{k}A_{k}}\right) \right) 
\\
&\leq &2d~\exp \left( 2\theta ^{2}\left\Vert \sum_{k}A_{k}^{2}\right\Vert
\right) \leq 2d~\exp \left( 2\theta ^{2}\left\Vert A\right\Vert \right) .
\end{eqnarray*}%
The second inequality comes from $\mathbb{E}_{\sigma }e^{2\theta \sigma
_{k}A_{k}}=\cosh \left( 2\theta A_{k}\right) \preceq e^{2\theta
^{2}A_{k}^{2}}$, and the fact that for positive operators $\lambda _{\max
}\,\ $and the norm coincide. The last inequality follows from the
implications $0\preceq A_{k}\preceq I$ $\implies $ $A_{k}^{2}\preceq A_{k}$ $%
\implies $ $\sum_{k}A_{k}^{2}\preceq \sum_{k}A_{k}$ $\implies $ $\left\Vert
\sum_{k}A_{k}^{2}\right\Vert \leq \left\Vert A\right\Vert $. Now we take the
expectation in $A_{1},\dots,A_{N}$. Together with the previous inequalities we
obtain 
\begin{align*}
\mathbb{E}e^{\theta \left\Vert A-\mathbb{E}A\right\Vert }& \leq 2d\mathbb{E}%
e^{2\theta ^{2}\left\Vert A\right\Vert }\leq 2d\mathbb{E}e^{2\theta
^{2}\left\Vert A-\mathbb{E}A\right\Vert }e^{2\theta ^{2}\left\Vert \mathbb{E}%
A\right\Vert } \\
& \leq 2d\left( \mathbb{E}e^{\theta \left\Vert A-\mathbb{E}A\right\Vert
}\right) ^{2\theta }e^{2\theta ^{2}\left\Vert \mathbb{E}A\right\Vert }.
\end{align*}%
The last inequality holds by Jensen's inequality since $\theta <1/4<1/2$.
Dividing by $\left( \mathbb{E}\exp \left( \theta \left\Vert A-\mathbb{E}%
A\right\Vert \right) \right) ^{2\theta }$, taking the power of $1/\left(
1-2\theta \right) $ and multiplying with $e^{\theta s}$ gives%
\begin{equation*}
\Pr \left\{ \left\Vert A-\mathbb{E}A\right\Vert >s\right\} \leq e^{-\theta s}%
\mathbb{E}e^{\theta \left\Vert A-\mathbb{E}A\right\Vert }\leq \left(
2d\right) ^{1/\left( 1-2\theta \right) }\exp \left( \frac{2\theta ^{2}}{%
1-2\theta }\left\Vert \mathbb{E}A\right\Vert -\theta s\right) .
\end{equation*}%
Since $\theta <1/4$, we have $\left( 2d\right) ^{1/\left( 1-2\theta \right)
}<\left( 2d\right) ^{2}$. Substitution of $\theta =s/\left( 6\left\Vert 
\mathbb{E}A\right\Vert +4s\right) <1/4$ together with some simplifications
gives (i).

It follows from elementary algebra that for $\delta >0$ with probability at
least $1-\delta $ we have%
\begin{eqnarray*}
\left\Vert A\right\Vert  &\leq &\left\Vert \mathbb{E}A\right\Vert +2\sqrt{%
\left\Vert \mathbb{E}A\right\Vert }\sqrt{\frac{9}{4}\ln \left( 4d^{2}/\delta
\right) }+6\ln \left( 4d^{2}/\delta \right)  \\
&\leq &p\left\Vert \mathbb{E}A\right\Vert +6q\ln \left( 4d^{2}/\delta
\right) ,
\end{eqnarray*}%
where the last line follows from $\left( 9/4\right) <6$ and Lemma \ref{Lemma
conjugate exponents} (iii). Equating the second term in the last line to $s$
and solving for the probability $\delta $ we obtain (ii), and (iii) follows
from Lemma \ref{Lemma partial integration}.
\end{proof}

\section{Proof of Theorem \protect\ref{Theorem Main}\label{Section Proof of
theorem 1}}

We prove the excess risk bound for heterogeneous sample sizes with the
weighted trace norm as in Remark \ref{Remark different sample sizes}
following the statement of Theorem \ref{Theorem Main}. The sample size for
the $n$-th task is thus $n_{t}$ and we abbreviate $\bar{n}$ for the average
sample size, $\bar{n}=\left( 1/T\right) \sum_{t}n_{t}$, so that $\bar{n}T$
is the total number of examples. The class of linear maps $W$ considered is%
\begin{equation*}
\mathcal{W=}\left\{ W\in 
\mathbb{R}
^{dT}:\left\Vert SW\right\Vert _{1}\leq B\sqrt{T}\right\} ,
\end{equation*}%
with $S=\left( s_{1},\dots,s_{T}\right) $ and $s_{t}=\sqrt{\bar{n}/n_{t}}$.
With $\mathcal{W}$ so defined we will prove the inequalities in Theorem \ref%
{Theorem Main} with $n$ replaced by $\bar{n}$. The result then reduces to
Theorem 1 if all the sample sizes are equal.

The first steps in the proof follow a standard pattern. We write%
\begin{multline*}
R\left( \hat{W}\right) -R\left( W^{\ast }\right)  \\
=\left[ R\left( \hat{W}\right) -\hat{R}\left( \hat{W},\mathbf{\bar{Z}}%
\right) \right] +\left[ \hat{R}\left( \hat{W},\mathbf{\bar{Z}}\right) -\hat{R%
}\left( W^{\ast },\mathbf{\bar{Z}}\right) \right] +\left[ \hat{R}\left(
W^{\ast },\mathbf{\bar{Z}}\right) -R\left( W^{\ast }\right) \right] .
\end{multline*}%
The second term is always negative by the definition of $\hat{W}$. The third
term depends only on $W^{\ast }$. Using Hoeffding's inequality \cite%
{Hoeffding 1963} it can be bounded with probability at least $1-\delta $ by $%
\sqrt{\ln \left( 1/\delta \right) /\left( 2\bar{n}T\right) }$. There remains
the first term which we bound by%
\begin{equation*}
\sup_{W\in \mathcal{W}}R\left( W\right) -\hat{R}\left( W\right) .
\end{equation*}%
It has by now become a standard technique (see \cite{Bartlett 2002}) to show
that this quantity is with probability at least $1-\delta $ bounded by 
\begin{equation}
\mathbb{E}_{\mathbf{\bar{Z}}}\mathcal{R}\left( \mathcal{W},\mathbf{\bar{Z}}%
\right) +\sqrt{\frac{\ln \left( 1/\delta \right) }{2\bar{n}T}}
\label{Rademacher bound}
\end{equation}%
or%
\begin{equation}
\mathcal{R}\left( \mathcal{W},\mathbf{\bar{Z}}\right) +\sqrt{\frac{9\ln
\left( 2/\delta \right) }{2\bar{n}T}},  \label{empirical Rademacher bound}
\end{equation}%
where the empirical Rademacher complexity $\mathcal{R}\left( \mathcal{W},%
\mathbf{\bar{Z}}\right) $ is defined for a multisample $\mathbf{\bar{Z}}$
with values in $\left( H\times 
\mathbb{R}
\right) ^{nT}$by%
\begin{equation*}
\mathcal{R}\left( \mathcal{W},\mathbf{\bar{Z}}\right) =\frac{2}{T}\mathbb{E}%
_{\mathbf{\sigma }}\sup_{W\in \mathcal{W}}\sum_{t=1}^{T}\frac{1}{n_{t}}%
\sum_{i=1}^{n_{t}}\sigma _{i}^{t}\ell \left( \left\langle
w_{t},X_{i}^{t}\right\rangle ,Y_{i}^{t}\right) .
\end{equation*}%
Standard results on Rademacher averages allow us to eliminate the Lipschitz
loss functions and give us%
\begin{eqnarray*}
\mathcal{R}\left( \mathcal{W},\mathbf{\bar{z}}\right)  &\leq &\frac{2L}{T}%
\mathbb{E}_{\sigma }\sup_{W\in \mathcal{W}}\sum_{t=1}^{T}\sum_{i=1}^{n_{t}}%
\sigma _{i}^{t}\left\langle w_{t},X_{i}^{t}/n_{t}\right\rangle  \\
&=&\frac{2L}{T}\mathbb{E}_{\sigma }\sup_{W\in \mathcal{W}}tr\left( W^{\ast
}D\right) =\frac{2L}{T}\mathbb{E}_{\sigma }\sup_{W\in \mathcal{W}}tr\left(
W^{\ast }SS^{-1}D\right) ,
\end{eqnarray*}%
where the random operator $D:H\rightarrow 
\mathbb{R}
^{T}$ is defined for $v\in H$ by $\left( Dv\right) _{t}=\left\langle
v,\sum_{i=1}^{n_{t}}\sigma _{i}^{t}X_{i}^{t}/n_{t}\right\rangle $, and the
diagonal matrix $S$ is as above. H\"{o}lder's and Jensen's inequalities give%
\begin{eqnarray*}
\mathcal{R}\left( \mathcal{W},\mathbf{\bar{Z}}\right)  &\leq &\frac{2L}{T}%
\sup_{W\in \mathcal{W}}\left\Vert SW\right\Vert _{1}\mathbb{E}_{\sigma
}\left\Vert S^{-1}D\right\Vert _{\infty }=\frac{2LB}{n\sqrt{T}}\mathbb{E}%
_{\sigma }\left\Vert S^{-1}D\right\Vert _{\infty } \\
&\leq &\frac{2LB}{\sqrt{T}}\sqrt{\mathbb{E}_{\sigma }\left\Vert D^{\ast
}S^{-2}D\right\Vert _{\infty }}.
\end{eqnarray*}%
Let $V_{t}$ be the random vector $V_{t}=\sum_{i=1}^{n_{t}}\sigma
_{i}^{t}X_{i}^{t}/\left( s_{t}n_{t}\right) $ and recall that the induced
rank-one operator $Q_{V_{t}}$ is defined by $Q_{V_{t}}v=\left\langle
v,V_{t}\right\rangle V_{t}=\left( 1/\left( \bar{n}n_{t}\right) \right)
\sum_{ij}\left\langle v,\sigma _{i}^{t}X_{i}^{t}\right\rangle \sigma
_{j}^{t}X_{j}^{t}$. Then $D^{\ast }S^{-2}D=\sum_{t=1}^{T}Q_{V_{t}}$, so we
obtain 
\begin{equation*}
\mathcal{R}\left( \mathcal{W},\mathbf{\bar{Z}}\right) \leq \frac{2LB}{\sqrt{T%
}}\sqrt{\mathbb{E}_{\sigma }\left\Vert \sum_{t}Q_{V_{t}}\right\Vert _{\infty
}}
\end{equation*}%
as the central object which needs to be bounded.

Observe that the range of any $Q_{V_{t}}$ lies in the subspace 
\begin{equation*}
M=Span\left\{ X_{i}^{t}:1\leq t\leq T\text{ and }1\leq i\leq n_{t}\right\} 
\end{equation*}%
which has dimension $\dim M\leq \bar{n}T<\infty $. We can therefore pull the
expectation inside the norm using Theorem \ref{Theorem Main Tool} if we can
verify a subexponential bound (\ref{Subexponential assumption}) on the
moments of the $Q_{V_{t}}$. This is the content of the following lemma.

\begin{lemma}
\label{Lemma Subexbound}Let $x_{1},\dots,x_{n}$ be in $H$ and satisfy $%
\left\Vert x_{i}\right\Vert \leq b$. Define a random vector by $%
V=\sum_{i}\sigma _{i}x_{i}$. Then for $m\geq 1$%
\begin{equation*}
\mathbb{E}\left[ \left( Q_{V}\right) ^{m}\right] \preceq m!\left(
2nb^{2}\right) ^{m-1}\mathbb{E}\left[ Q_{V}\right] .
\end{equation*}
\end{lemma}

\begin{proof}
Let $K_{m,n}$ be the set of all sequences $\left( j_{1},\dots,j_{2m}\right) $
with $j_{k}\in \left\{ 1,\dots,n\right\} $, such that each integer in $\left\{
1,\dots,n\right\} $ occurs an even number of times. It is easily shown by
induction that the number of sequences in $K_{m,n}$ is bounded by 
\begin{equation*}
\left\vert K_{m,n}\right\vert \leq \left( 2m-1\right) !!n^{m},
\end{equation*}%
where $\left( 2m-1\right) !!=\prod_{i=1}^{m}\left( 2i-1\right) \leq
m!2^{m-1} $.

Now let $v\in H$ be arbitrary. By the definition of $V$ and $Q_{V}$ we have
for any $v\in H$ that%
\begin{equation*}
\left\langle \mathbb{E}\left[ \left( Q_{V}\right) ^{m}\right]
v,v\right\rangle =\sum_{j_{1},\dots,j_{2m}=1}^{n}\mathbb{E}\left[ \sigma
_{j_{1}}\sigma _{j_{2}}\cdots\sigma _{j_{2m}}\right] \left\langle
v,x_{j_{1}}\right\rangle \left\langle x_{j_{2}},x_{j_{3}}\right\rangle
\dots\left\langle x_{j_{2m}},v\right\rangle .
\end{equation*}%
The properties of independent Rademacher variables imply that $\mathbb{E}%
\left[ \sigma _{j_{1}}\sigma _{j_{2}}\cdots\sigma _{j_{2m}}\right] =1$ if $j\in
K_{m,n}$ and zero otherwise. For $m=1$ this shows $\left\langle \mathbb{E}%
\left[ \left( Q_{V}\right) ^{m}\right] v,v\right\rangle =\left\langle 
\mathbb{E}\left[ Q_{V}\right] v,v\right\rangle =\sum_{j}\left\langle
v,x_{j}\right\rangle ^{2}$. For $m>1$, since $\left\Vert x_{i}\right\Vert
\leq b$ and by two applications of the Cauchy-Schwarz inequality 
\begin{eqnarray*}
\left\langle \mathbb{E}\left[ \left( Q_{V}\right) ^{m}\right]
v,v\right\rangle  &=&\sum_{\mathbf{j}\in K_{m,n}}\left\langle
v,x_{j_{1}}\right\rangle \left\langle x_{j_{2}},x_{j_{3}}\right\rangle
\cdots\left\langle x_{j_{2m}},v\right\rangle  \\
&\leq &b^{2\left( m-1\right) }\sum_{\mathbf{j}\in K_{m,n}}\left\vert
\left\langle v,x_{j_{1}}\right\rangle \right\vert \left\vert \left\langle
x_{j_{2m}},v\right\rangle \right\vert  \\
&\leq &b^{2\left( m-1\right) }\left( \sum_{\mathbf{j}\in
K_{m,n}}\left\langle v,x_{j_{1}}\right\rangle ^{2}\right) ^{1/2}\left( \sum_{%
\mathbf{j}\in K_{m,n}}\left\langle v,x_{j_{2m}}\right\rangle ^{2}\right)
^{1/2} \\
&=&b^{2\left( m-1\right) }\sum_{j}\left\langle v,x_{j}\right\rangle
^{2}.\sum_{\mathbf{j}\in K_{m,n}\text{ such that }j_{1}=j}1 \\
&=&\left\langle \mathbb{E}\left[ Q_{V}^{m}\right] v,v\right\rangle \times
\left( 2m-1\right) !!n^{m-1}b^{2\left( m-1\right) } \\
&\leq &m!\left( 2nb^{2}\right) ^{m-1}\left\langle \mathbb{E}\left[ Q_{V}%
\right] v,v\right\rangle .
\end{eqnarray*}%
The conclusion follows since for self-adjoint matrices $\left( \forall
v,\left\langle Av,v\right\rangle \leq \left\langle Bv,v\right\rangle \right)
\implies A\preceq B$.\bigskip 
\end{proof}

If we apply this lemma to the vectors $V_{t}$ defined above with $b=1/\left(
s_{t}n_{t}\right) $, using $s_{t}^{2}n_{t}=\bar{n}$, we obtain%
\begin{equation*}
\mathbb{E}\left[ \left( Q_{Vt}\right) ^{m}\right] \preceq m!\left( \frac{2}{%
s_{t}^{2}n_{t}}\right) ^{m-1}\mathbb{E}\left[ Q_{V}\right] =m!\left( \frac{2%
}{\bar{n}}\right) ^{m-1}\mathbb{E}\left[ Q_{V}\right] ,.
\end{equation*}%
Applying the last conclusion of Theorem \ref{Theorem Main Tool} with $R=2/%
\bar{n}$ and $d=\bar{n}T$ now yields%
\begin{equation*}
\sqrt{\mathbb{E}_{\sigma }\left\Vert \sum Q_{V_{t}}\right\Vert _{\infty }}%
\leq \sqrt{\left\Vert \sum \mathbb{E}_{\sigma }Q_{V_{t}}\right\Vert _{\infty
}}+\sqrt{\frac{2}{\bar{n}}\left( \ln \left( \bar{n}T\right) +1\right) },
\end{equation*}%
and since $\sum_{t}\mathbb{E}_{\sigma }Q_{V_{t}}=\sum_{t}\left( 1/\bar{n}%
\right) \sum_{i}\left( 1/n_{t}\right) Q_{X_{i}^{t}}=T\hat{C}/\bar{n}$ we get%
\begin{eqnarray}
\mathcal{R}\left( \mathcal{W},\mathbf{\bar{Z}}\right)  &\leq &\frac{2LB}{%
\sqrt{T}}\sqrt{\mathbb{E}_{\sigma }\left\Vert \sum_{t}Q_{V_{t}}\right\Vert
_{\infty }}  \notag \\
&\leq &2LB\left( \sqrt{\frac{\left\Vert \hat{C}\right\Vert _{\infty }}{\bar{n%
}}}+\sqrt{\frac{2\left( \ln \left( \bar{n}T\right) +1\right) }{\bar{n}T}}%
\right) .  \label{Empirical bound}
\end{eqnarray}%
Together with (\ref{empirical Rademacher bound}) and the initial remarks in
this section this proves the second part of Theorem \ref{Theorem Main}.

To obtain the first assertion we take the expectation of (\ref{Empirical
bound}) and use Jensen's inequality, which then confronts us with the
problem of bounding $\mathbb{E}\left\Vert \hat{C}\right\Vert _{\infty }$ in
terms of $\left\Vert C\right\Vert _{\infty }=\left\Vert \mathbb{E}\hat{C}%
\right\Vert _{\infty }$. Note that $\bar{n}T\hat{C}=\sum_{t}%
\sum_{i=1}^{n_{t}}Q_{X_{i}^{t}}$. Here Theorem \ref{Theorem Main Tool}
doesn't help because the covariance may have infinite rank, so that we
cannot find a finite dimensional subspace containing the ranges of all the $%
Q_{X_{i}^{t}}$. But since $\left\Vert X_{i}^{t}\right\Vert \leq 1$ all the $%
Q_{X_{i}^{t}}$ satisfy $0\preceq Q_{X_{i}^{t}}\preceq I$ and are rank-one
operators, we can invoke Theorem \ref{Theorem Oliveira} with $d=\bar{n}T$.
This gives 
\begin{equation*}
\sqrt{\mathbb{E}\left\Vert \hat{C}\right\Vert }\leq \sqrt{\left\Vert
C\right\Vert }+\sqrt{\frac{6\left( \ln \left( 4\bar{n}T\right) +1\right) }{%
\bar{n}T}},
\end{equation*}%
and from (\ref{Empirical bound}) and Jensen's inequality and some
simplifications we obtain%
\begin{eqnarray*}
\mathbb{E}\mathcal{R}\left( \mathcal{W},\mathbf{\bar{Z}}\right)  &\leq
&2LB\left( \sqrt{\frac{\mathbb{E}\left\Vert \hat{C}\right\Vert _{\infty }}{%
\bar{n}}}+\sqrt{\frac{2\left( \ln \left( \bar{n}T\right) +1\right) }{\bar{n}T%
}}\right)  \\
&\leq &2LB\left( \sqrt{\frac{\left\Vert C\right\Vert _{\infty }}{\bar{n}}}+5%
\sqrt{\frac{\ln \left( nT\right) +1}{\bar{n}T}}\right) ,
\end{eqnarray*}%
which, together with (\ref{Rademacher bound}), gives the first assertion of
Theorem \ref{Theorem Main}.\bigskip 

A similar application of Theorem \ref{Theorem Oliveira} applied to the bound
(\ref{Kakade bound}) in \cite{KakadeEtAl 2012} yields the bound (\ref{Kakade
bound distribution dependent}).

\end{document}